\documentclass{article}

\usepackage{microtype}
\usepackage{graphicx}
\usepackage{subfigure}
\usepackage{booktabs} 

\usepackage{hyperref}

\DeclareGraphicsExtensions{.pdf,.jpeg,.png}
\usepackage{amsmath,amssymb,amsthm}

\newcommand{\Db}{{\mathbf D}}

\newcommand{\Rb}{{\mathbf R}}

\newcommand{\Wb}{{\mathbf W}}

\newcommand{\Bc}{{\mathcal B}}

\newcommand{\Dc}{{\mathcal D}}
\newcommand{\Ec}{{\mathcal E}}

\newcommand{\Xbc}{\boldsymbol{\mathcal X}}
\newcommand{\Ybc}{\boldsymbol{\mathcal Y}}
\newcommand{\Zbc}{\boldsymbol{\mathcal Z}}

\newcommand{\Wbc}{\boldsymbol{\mathcal W}}

\newcommand{\Rd}{{\mathbb R}}

\newcommand{\hank}{\mathbb{H}}

\newtheorem{theorem}{Theorem}
\newtheorem{corollary}[theorem]{Corollary}
\newtheorem{lemma}[theorem]{Lemma}
\newtheorem{proposition}[theorem]{Proposition}

\newcommand{\beq}{\begin{equation}}
\newcommand{\eeq}{\end{equation}}
\newcommand{\beqa}{\begin{eqnarray}}
\newcommand{\eeqa}{\end{eqnarray}}

\usepackage[accepted]{icml2019} 

\icmltitlerunning{Understanding Geometry of Encoder-Decoder CNNs}

\begin{document}

\twocolumn[
\icmltitle{Understanding Geometry of Encoder-Decoder CNNs}



\icmlsetsymbol{equal}{*}

\begin{icmlauthorlist}
\icmlauthor{Jong Chul Ye}{to,math}
\icmlauthor{Woon Kyoung Sung}{math}
\end{icmlauthorlist}

\icmlaffiliation{to}{Dept. of Bio/Brain Engineering, KAIST
 Daejeon 34141,
Republic of Korea.
}
\icmlaffiliation{math}{Dept. of Mathematical Sciences, KAIST, 
 Daejeon 34141,
Republic of Korea.}

\icmlcorrespondingauthor{Jong Chul Ye}{jong.ye@gmail.com}

\icmlkeywords{Encoder-decoder network,  expressivity, generalizability, optimization landscape, convolution framelets}

\vskip 0.3in
]



\printAffiliationsAndNotice{}  

\begin{abstract}
Encoder-decoder networks using convolutional neural network (CNN) architecture have been extensively used in deep learning literatures
thanks to its excellent performance for various inverse problems. 
However, 
it is still difficult to obtain coherent geometric view why such an architecture gives the desired performance.
Inspired by recent theoretical understanding on  {generalizability}, {expressivity}  and {optimization landscape} of neural networks, as well as the theory of deep 
convolutional framelets,
here we provide a unified theoretical framework  that leads to a better understanding of geometry of encoder-decoder CNNs.
Our unified framework shows that encoder-decoder CNN architecture is closely related to nonlinear frame  representation
using  combinatorial convolution frames, whose expressibility increases exponentially with the  depth.
We also 
demonstrate  the importance of skipped connection 
 in terms
of expressibility,  and  optimization landscape. 
		\vspace*{-0.5cm}
\end{abstract}

\section{Introduction}
\label{submission}

For the last decade, we have witnessed the unprecedented success of deep neural networks (DNN) in various applications in computer vision,
classification, medical imaging, etc.
Aside from traditional applications such as classification \cite{krizhevsky2012imagenet}, segmentation \cite{ronneberger2015u},
image denoising \cite{zhang2017beyond}, super-resolution \cite{kim2016accurate}, etc,
 deep learning approaches have already become the state-of-the-art technologies in various inverse problems in x-ray CT, MRI, etc \cite{kang2016deep,
jin2017deep,hammernik2018learning}

However, the more we see the success of deep learning, the more mysterious the nature of deep neural networks becomes.
In particular, the amazing aspects of {\em expressive power, generalization capability}, and  {\em optimization landscape} of DNNs have become an intellectual
challenge for  machine learning community, leading to many new theoretical results with varying capacities to facilitate the understanding of deep neural networks \cite{ge2017optimization,hanin2017relu,yarotsky2017error,nguyen2017loss,arora2016understanding,du2018gradient,raghu2017expressive,bartlett2017spectrally,neyshabur2018towards,nguyen2018optimization,rolnick2017power,shen2018differential}.

 \begin{figure*}[!hbt]
	\centerline{\includegraphics[width=14cm,height=4cm]{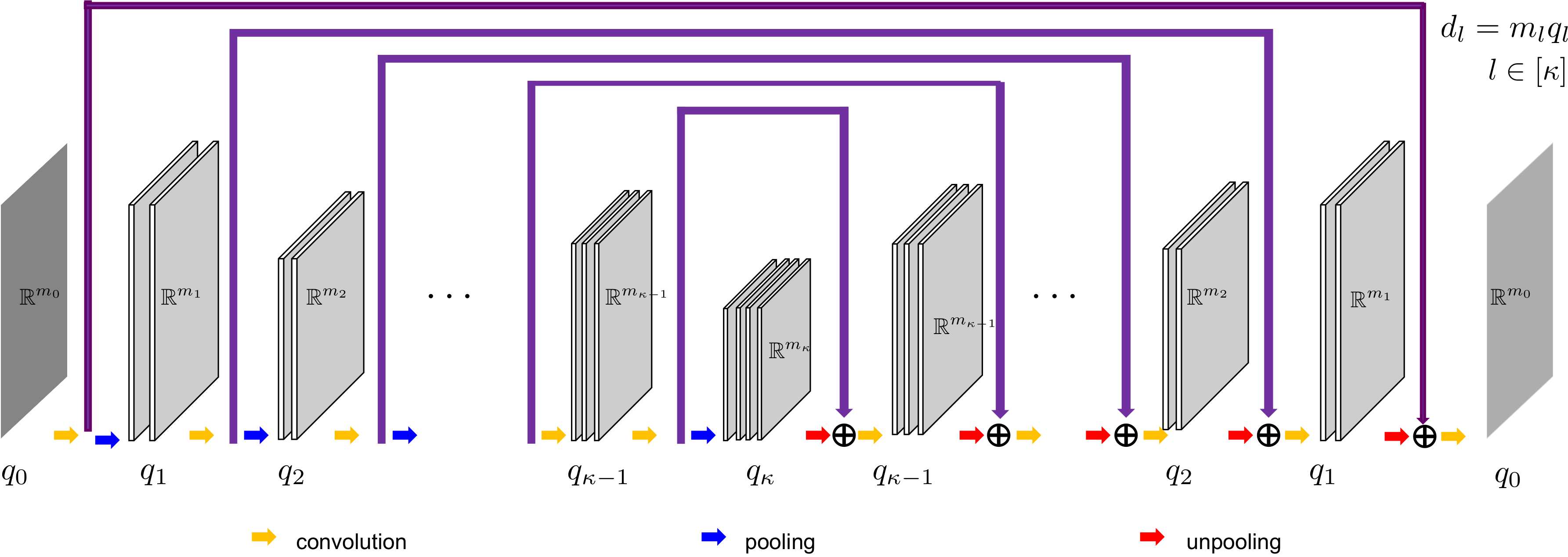}}
	\vspace*{-0.2cm}
	\caption{An architecture of $\kappa$-layer symmetric encoder-decoder CNN with skipped connections. Here, $q_l$ denotes the number of channels at the $l$-th layer, whereas $m_l$ refers to each channel dimension, and $d_l$ represents the total
	dimension of the feature at the $l$-th layer.}
		\vspace*{-0.5cm}
	\label{fig:network}
\end{figure*}

In inverse problems, one of the most widely employed network architectures is so-called encoder-decoder CNN 
architectures  
 \cite{ronneberger2015u}. In contrast to the simplified form of the neural
networks that are often used in  theoretical analysis,  these encoder-decoder CNNs usually have more complicated network
architectures such as symmetric network configuration,
 skipped connections, etc.
Therefore, it is not clear how the aforementioned theory can be used to understand the geometry of  encoder-decoder CNNs to examine the origin of their superior performance.

Recently, 
the authors in \cite{ye2017deep} proposed so-called deep convolutional framelets to explain the encoder-decoder CNN architecture
 from  a signal processing perspective.
The main idea is that a data-driven decomposition of Hankel matrix constructed from the input data 
provides  encoder-decoder layers that have striking similarity to the encoder-decoder CNNs.
However, one of the main weaknesses of the theory is that it is not clear where the exponential expressiveness comes from.
Moreover,  many theoretical issues of neural networks such as  generalizability and the optimization landscape, which have been extensively studied in machine learning literature,  have not been 
addressed.  

Therefore,  this work aims at filling the gap and finding the connections between machine learning and signal processing 
to provide a unified theoretical analysis  
that facilitates the geometric understanding of
encoder-decoder CNNs.
Accordingly, we  have revealed the following geometric features of encoder-decoder CNNs:
 \begin{itemize}
 \item  An encoder-decoder  CNN with an over-parameterized feature layer approximates  a map between two smooth manifolds that
 is decomposed as a high-dimensional embedding followed by a quotient map.
  \item An encoder-decoder CNN with ReLU nonlinearity can be understood as deep convolutional framelets that use combinatorial frames of spatially varying convolutions. Accordingly, the number of linear representations increases exponentially with the network depth.
This also suggests that the input space is divided into non-overlapping areas where each area shares the common linear representation.
\item We derive  an explicit form of the Lipschitz condition that determines the generalization capability of the encoder-decoder CNNs.  The expression shows  that 
the expressiveness of the network is not affected by the control of the Lipschitz constant.

 \item We provide explicit conditions under which the optimization landscape for  encoder-decoder CNNs is benign.
Specifically, we show that the skipped connection play important roles in smoothing out the optimization landscape.
 \end{itemize}

All the proof of the theorems and lemmas in this paper are  included in the Supplementary  Material.

 \section{Related Works}

%
Choromanska et al \cite{choromanska2015loss} employed the spin glass model
from statistical physics to analyze the representation power of deep neural networks.
Telgarsky constructs
interesting  classes of  functions that can be only computed efficiently
by deep ReLU nets, but not by shallower
networks with a similar number of parameters \cite{telgarsky2016colt}.
Arora et al \cite{arora2016understanding} showed that
for every natural number $k$ there exists a ReLU network with $k^2$ hidden layers and total size of
$k^2$, which can be represented by $\frac{1}{2}k^{k+1}-1$ neurons with at most $k$-hidden layers. 
All these results agree that the expressive power of deep neural networks increases exponentially with the network depth.

The generalization capability have been addressed in terms of various complexity measures
such as Rademacher complexity \cite{bartlett2002rademacher},  VC bound \cite{anthony2009neural}, 
Kolmorogov complexity \cite{schmidhuber1997discovering}, etc.
However, a recent work \cite{zhang2016understanding} showed intriguing results that these classical bounds are too pessimistic to explain the generalizability
 of  deep neural networks. Moreover,  it has been repeatedly shown that
over-parameterized  deep neural networks, which are trained with fewer samples than the number of neurons,
generalize well rather than overfitting \cite{cohen2018dnn,wei2018margin,brutzkus2017sgd,du2018power}, which phenomenon cannot be explained by 
 the classical complexity results.
%

The optimization landscape of neural networks have been another important theoretical
issue in neural networks.  Originally observed in  linear deep neural networks \cite{kawaguch2016nips}, the
benign optimization landscape  has been consistently observed
in various neural networks \cite{du2018gradient,nguyen2018optimization,du2017gradient,nguyen2017loss}.

However, these theoretical works mainly focus on simplified network architectures, and we are not aware of analysis for encoder-decoder CNNs.
%

\section{Encoder-Decoder CNNs}

    %

\subsection{Definition}

In this section,  we provide a formal definition 
of  encoder-decoder CNNs (E-D CNNs) to facilitate the theoretical analysis.
Although our definition is for 1-dimensional signals, its extension to 2-D images
is straightforward.

\subsubsection{Basic Architecture}

Consider  encoder-decoder networks in Fig.~\ref{fig:network}.
Specifically, the encoder network maps a given input signal $x\in\Xbc\subset \Rd^{d_0}$ to a 
 feature space $z \in \Zbc\subset \Rd^{d_\kappa}$, whereas the decoder
 takes this feature map  as an input, process it  and produce an output
$y \in \Ybc\subset \Rd^{d_L}$.
In this paper, symmetric configuration is considered so that
both encoder and decoder have the same number of layers, say $\kappa$;
the input and output dimensions for the encoder layer $\Ec^l$ and the decoder layer $\Dc^l$ are symmetric:
\begin{eqnarray*}
\Ec^l:\Rd^{d_{l-1}} \mapsto \Rd^{d_l},  \quad 
\Dc^l:\Rd^{d_{l}} \mapsto \Rd^{d_{l-1}}
\end{eqnarray*}
where $l\in [\kappa]$ with $[n]$ denoting the set $\{1,\cdots, n\}$;
and both input and output dimension is $d_0$.
More specifically, 
the $l$-th layer input signal for the encoder layer comes from $q_{l-1}$ number of input channels:
$$\xi^{l-1}=\begin{bmatrix} \xi_1^{l-1\top} & \cdots & \xi^{l-1\top}_{q_{l-1}} \end{bmatrix}^\top \in   \Rd^{d_{l-1}}, \quad  $$ 
where  $^\top$ denotes the transpose,
and $\xi_j^{l-1} \in \Rd^{m_{l-1}}$ refers to the $j$-th channel input with the dimension $m_{l-1}$.
Therefore, the overall input dimension is given by $d_{l-1}:= {m_{l-1}q_{l-1}}$.
Then,  the $l$-th layer encoder generates  $q_l$  channel output  using the convolution operation:
 \begin{eqnarray}\label{eq:encConv}
\xi_j^l = \sigma\left(\Phi^{l\top} \sum_{k=1}^{q_{l-1}}\left(\xi_k^{l-1}\circledast \overline \psi_{j,k}^l\right)\right) ,~j\in [q_l] 
\end{eqnarray}
where 
 $\xi_j^l \in \Rd^{m_l}$ refers to the $j$-th channel output after the
convolutional filtering with the $r$-tap filters $\overline\psi_{j,k}^l\in \Rd^r$  and pooling operation $\Phi^{l\top} \in \Rd^{m_l \times m_{l-1}}$, and 
 $\sigma(\cdot)$ denotes the element wise rectified linear unit (ReLU).
  More specifically, $\overline\psi_{j,k}^l\in \Rd^r$ denotes the $r$-tap convolutional kernel that is convolved with the $k$-th input  to contribute
  to the output of the $j$-th channel,  $\circledast$ is the circular convolution  via periodic boundary condition to avoid special treatment of the convolution at the boundary,
  and $\overline v$ refers to the flipped version of the vector $v$.    For the formal definition of the convolution operation used in this paper,  see Appendix~A in Supplementary Material. 

Moreover, as shown in Appendix~B in Supplementary Material, 
an equivalent matrix representation of  the encoder layer is then given by
\begin{eqnarray*}
 \xi^l:=\sigma(E^{l\top} \xi^{l-1})=\begin{bmatrix} \xi^{l\top}_1 & \cdots & \xi^{l\top}_{q_{l}} \end{bmatrix}^\top
\end{eqnarray*}
where
 $E^l \in \Rd^{d_{l-1}\times d_{l}}$   is computed by\footnote{Here, without loss of generality, bias term is not explicitly shown, since it can be incorporated into the matrix $E^l$ and $D^l$ as an additional column.}
  \begin{eqnarray}\label{eq:El}
 E^l= \begin{bmatrix} 
\Phi^l\circledast \psi^l_{1,1} & \cdots &  \Phi^l\circledast \psi^l_{q_l,1}  \\
  \vdots & \ddots & \vdots \\
\Phi^l\circledast \psi^l_{1,q_{l-1}} & \cdots &  \Phi^l\circledast \psi^l_{q_{l},q_{l-1}}
 \end{bmatrix}
 \end{eqnarray}
with 
\begin{eqnarray}
\begin{bmatrix} \Phi^l \circledast \psi_{i,j}^l  \end{bmatrix} :=\begin{bmatrix} \phi^l_1 \circledast \psi_{i,j}^l & \cdots & \phi^l_{m_l} \circledast  \psi_{i,j}^l\end{bmatrix}  \label{eq:defconv}
\end{eqnarray}

 On the other hand,
 the $l$-th layer input signal for the decoder layer comes from $q_{l}$ channel inputs, i.e. 
$\tilde\xi^{l}=\begin{bmatrix}\tilde\xi_1^{l\top} & \cdots & \tilde\xi^{l}_{q_{l}\top} \end{bmatrix}^\top \in   \Rd^{d_{l}},$
 and
 the decoder layer convolution is given by
  \begin{eqnarray}\label{eq:decConv}
\tilde\xi_j^{l-1} = \sigma\left(\sum_{k=1}^{q_{l}}\left(\tilde\Phi^l\tilde\xi^{l}_k\circledast  {\tilde\psi_{j,k}^l}\right)\right) ,\quad j\in [q_{l-1}] 
\end{eqnarray}
where  the unpooling layer is denoted by $\tilde\Phi^l \in \Rd^{m_{l-1}\times m_{l}}$.
Note that \eqref{eq:encConv} and \eqref{eq:decConv} differ in their
order of the pooling or unpooling layers.  Specifically, a pooling operation is applied after the convolution at the encoder layer, whereas, at the decoder,
an unpooling operation is performed  before the convolution  to maintain the symmetry of the networks.
In matrix form, 
a decoder layer is given by
\begin{eqnarray*}
\tilde \xi^{l-1}:=\sigma(D^l \tilde\xi^{l})=\begin{bmatrix}  \tilde\xi^{l-1\top}_1 &\cdots &
\tilde\xi^{l-1\top}_{q_{l-1}} \end{bmatrix}^\top  \
\end{eqnarray*}
where 
$D^l \in \Rd^{d_{l}\times d_{l-1}}$   is computed by
  \begin{eqnarray}\label{eq:Dl}
 D^l= \begin{bmatrix} 
\tilde\Phi^l\circledast \tilde\psi^l_{1,1} & \cdots &  \tilde\Phi^l\circledast \tilde\psi^l_{1,q_l}  \\
  \vdots & \ddots & \vdots \\
\tilde\Phi^l\circledast \tilde\psi^l_{q_{l-1},1} & \cdots &  \tilde\Phi^l\circledast \tilde\psi^l_{q_{l-1},q_{l}}
 \end{bmatrix}
 \end{eqnarray}

\subsubsection{E-D CNN with skipped connection}

As shown in Fig.~\ref{fig:network},
a skipped connection is often used to bypass an encoder layer output to a decoder layer.
The corresponding filtering operation at the $l$-th layer encoder is  described by
 \begin{eqnarray}\label{eq:encConvSkip}
\begin{bmatrix}\xi_j^l \\ \chi_j^l \end{bmatrix}
= \begin{bmatrix} \sigma\left(\Phi^{l\top} \sum_{k=1}^{q_{l-1}}\left(\xi_k^{l-1}\circledast \overline \psi_{j,k}^l\right)\right) \\ 
 \sigma\left(\sum_{k=1}^{q_{l-1}}\left(\xi_k^{l-1}\circledast \overline \psi_{j,k}^l\right)\right) \end{bmatrix}
\end{eqnarray}
where $\chi_j^l$ and $\xi_j^l$  denote the skipped output,  and the pooled output via $\Phi^{l\top}$, respectively, after the filtering with
$\overline\psi_{j,k}$.
As shown in Fig.~\ref{fig:network}, 
the skipped branch is no more filtered at the subsequent layer, but is merged at the symmetric decoder layer:
 \begin{eqnarray*}
\tilde\xi_j^{l-1} = \sigma\left(\sum_{k=1}^{q_{l}}\left((\tilde\Phi^{l}\tilde\xi^{l}_k+\chi^l_k)\circledast {\tilde\psi_{j,k}^{l}}\right)\right) \end{eqnarray*}

In matrix form, the encoder layer with the skipped connection can be represented by
$$
\Ec^l: \xi^{l-1} \mapsto  \begin{bmatrix} \xi^{l\top} & \chi^{l\top} \end{bmatrix}^\top$$
where  
\begin{eqnarray}
\xi^l := \sigma(E^{l\top} \xi^{l-1}) &, & 
\chi^l := \sigma(S^{l\top} \xi^{l-1}) \label{eq:feature}
\end{eqnarray}
where $E^l$ is given in \eqref{eq:El} and
the skipped branch filter matrix $S^l$ is represented by
 \begin{eqnarray}\label{eq:Sl}
 S^l= \begin{bmatrix} 
I_{m_{l-1}}\circledast \psi^l_{1,1} & \cdots &  I_{m_{l-1}}\circledast \psi^l_{q_l,1}  \\
  \vdots & \ddots & \vdots \\
I_{m_{l-1}}\circledast \psi^l_{1,q_{l-1}} & \cdots &  I_{m_{l-1}}\circledast \psi^l_{q_{l},q_{l-1}}
 \end{bmatrix}
 \end{eqnarray}
 where $I_{m_{l-1}}$ denotes the $m_{l-1}\times m_{l-1}$ identity matrix.
This implies that  we can regard the skipped branch as the identity pooling $I_{m_{l-1}}$ applied to the filtered signals.
Here, we denote the output dimension of the skipped connection
as $$s_l:=m_{l-1} q_l \quad .$$
Then, the skipped branch at the $l$-th encoder layer is merged at the $l$-th decoder layer,
which  is defined as
$$ \Dc^{l}: \begin{bmatrix} \tilde\xi^{l\top} & \chi^{l\top} \end{bmatrix}^\top \mapsto \tilde\xi^{l-1}
$$
where
\begin{eqnarray}\label{eq:sum}
\tilde\xi^{l-1}
:=\sigma(D^{l} \tilde\xi^{l}+\tilde S^{l}\chi^l)
\end{eqnarray}
and $D^{l}$ is  defined in \eqref{eq:Dl}, and 
$\tilde S^l$  is given by
\begin{eqnarray}\label{eq:El2}
 \tilde S^l= \begin{bmatrix} 
I_{m_{l-1}}\circledast \tilde\psi^l_{1,1} & \cdots &  I_{m_{l-1}}\circledast \tilde\psi^l_{1,q_{l}}  \\
  \vdots & \ddots & \vdots \\
I_{m_{l-1}}\circledast \tilde\psi^l_{q_{l-1},1} & \cdots &  I_{m_{l-1}}\circledast \tilde\psi^l_{q_{l-1},q_{l}}
 \end{bmatrix}
 \end{eqnarray}

\subsection{Parameterization of E-D CNNs}
		\vspace*{-0.2cm}

At the $l$-th encoder (resp. decoder) layer, there are $q_lq_{l-1}$ filter set that generates the $q_l$  (resp. $q_{l-1}$) output channels
from $q_{l-1}$ (resp. $q_{l}$) input channels. In many CNNs, the filter lengths are set to equal across the layer.
In our case, we set this as $r$, so the number of filter coefficients for the $l$-layer is 
$$n_l:=rq_lq_{l-1}, \quad l\in[\kappa]$$
These parameters  should be estimated during the training phase. 
Specifically,
by denoting the set of all parameter matrices  $\Wbc=\Wbc_E\times \Wbc_D$
where $\Wbc_E:=\Rd^{n_\kappa}  \times \cdots \times \Rd^{n_{1}}$ 
and $\Wbc_D:=\Rd^{n_1} \times \cdots \times \Rd^{n_\kappa}$,
we compose all layer-wise maps to
define an  encoder-decoder CNN as
\begin{eqnarray}\label{eq:Fcnn}
 z = F(\Wb,x)   . 
 \end{eqnarray}
Regardless of the existence of skipped connections, note that the same number of unknown parameters is used because the skipped connection uses the same set of filters.

		\vspace*{-0.1cm}
\section{Theoretical Analysis of E-D CNNs}
		\vspace*{-0.1cm}


\subsection{Differential Topology }

First, we briefly revisit  the work  by Shen \cite{shen2018differential}, which gives an topological insight on the E-D
CNNs.

\begin{proposition}[Extension of Theorem~3 in \cite{shen2018differential}]\label{thm:embedding}
Let $f : \Xbc\mapsto \Ybc \subset \Rd^q$ be a continuous map of smooth manifolds  such that $f=g \circ h$, where
$g : \Rd^p \mapsto  \Rd^q$ with $p\geq q$ is a Lipschitz continuous map.  If $p > 2 \dim\Xbc$, then there
exists a smooth embedding $\tilde h: \Xbc \mapsto \Rd^p$, so that the following inequality holds true for
a chosen norm and all $x\in \Xbc$ and $\epsilon >0$:
$$\|f(x) -g\circ \tilde h(x)\|\leq \epsilon$$
\end{proposition}

Here, $p > 2 \dim\Xbc$ comes from the weak Whitney embedding theorem \cite{whitney1936differentiable,tu2011introduction}.
Note that
Theorem~\ref{thm:embedding} informs that a neural network,  designed as a continuous map of smooth manifolds,
 can be considered as
an approximation of a task map  that is composed of a smooth embedding followed by an additional
map.
In fact,  this decomposition is quite general for a map between smooth manifolds as shown in the following proposition:
 \begin{proposition}\label{thm:quotient}\cite{shen2018differential}
 Let $f : \Xbc\mapsto \Ybc\subset \Rd^q$ be a map of smooth manifolds, then the task $ f$ admits a decomposition of
$f = g\circ h$, where $ h: \Xbc \mapsto \Zbc \subset \Rd^p$ with $p \geq 2 \dim \Xbc$ is a smooth embedding.
Furthermore, the task map $f$ is a quotient map, if and only if the  map $g$ is a
quotient map.
\end{proposition}

To understand the meaning of the last sentence in Proposition~\ref{thm:quotient},
we briefly review the concept of the quotient space and quotient map \cite{tu2011introduction}.
Specifically, let $\sim$ be an equivalence relation on $\Xbc$.  Then, the quotient space, $\Ybc = \Xbc/\sim$ is defined to be the set of equivalence classes of elements of $\Xbc$.
For example,
 we can declare images perturbed by noises as an equivalent class such that
our quotient map is designed to map the noisy signals to its noiseless equivalent image.

It is remarkable that Proposition~\ref{thm:embedding} and Proposition~\ref{thm:quotient} 
give interpretable conditions for design parameters such as network width (i.e. no of channels), pooling layers, etc. For example, if there are no pooling layers, the dimensionality conditions in Proposition~\ref{thm:embedding} and Proposition~\ref{thm:quotient}
can be easily met in practice by increasing the number of channels more than twice the input channels.
 With the pooling layers, one could calculate the number of channels in a similar way.  
 In general, Proposition~\ref{thm:embedding} and Proposition~\ref{thm:quotient} 
strongly
suggest an encoder-decoder architecture with the constraint $d_0\leq d_1\leq \cdots \leq d_\kappa$ with $d_\kappa > 2 d_0$, where  an encoder  maps an input signal
to higher dimensional feature space whose dimension is at least twice bigger than the input space. 
Then, the decoder determines the nature of the overall neural network.


\subsection{Links to the frame representation}

One of the important contributions of recent theory of deep convolutional framelets \cite{ye2017deep} is that
encoder-decoder CNNs have an interesting link to multi-scale convolution framelet expansion.
To see this,  we first define filter matrices $\Psi^l \in \Rd^{rq_{l-1}\times q_{l}}$ and $\tilde\Psi^{l} \in \Rd^{rq_{l-1}\times  q_{l}}$
for encoder and decoder:
  $$ 
  \Psi^l :=
  \begin{bmatrix} 
 \psi^l_{1,1}& \cdots &   \psi^l_{q_l,1} \\
  \vdots & \ddots & \vdots \\
 \psi^l_{1,q_{l-1}} & \cdots &   \psi^l_{q_{l},q_{l-1}}  
 \end{bmatrix} 
 $$
 $$
 \tilde\Psi^{l} :=\begin{bmatrix} 
\tilde\psi^{l}_{1,1}& \cdots &  \tilde\psi^{l}_{1,q_{l}} \\
  \vdots & \ddots & \vdots \\
\tilde\psi^{l}_{q_{l-1},1} & \cdots &  \tilde\psi^{l}_{q_{l-1},q_{l}}  
 \end{bmatrix} 
 $$
Then, the following proposition, which is novel and significantly extended from \cite{ye2017deep}, states the 
importance of  the frame conditions for the pooling layers and filters to obtain convolution framelet expansion \cite{yin2017tale}.
\begin{proposition}\label{thm:PR}
Consider an encoder-decoder CNN without ReLU nonlinearities.
Let $\Phi^{l\top}$ and $\tilde\Phi^{l}$ denote the $l$-th encoder and decoder layer pooling  layers, respectively,
and $\Psi^l$ and $\tilde\Psi^{l}$ refer to the
encoder and decoder filter matrices. Then, the following statements are true.

1) For the encoder-decoder CNN without skipped connection,
if the following frame conditions are satisfied for all $l\in [\kappa]$
\begin{eqnarray}\label{eq:PRl}
 \tilde\Phi^{l}\Phi^{l\top}=\alpha I_{m_{l-1}},~
  \Psi^l \tilde\Psi^{l\top} =   \frac{1}{r\alpha}I_{rq_{l-1}}  
\end{eqnarray}
then we have
\begin{eqnarray}\label{eq:PR0}
x 
&=& \sum_{i} \langle b_i, x \rangle \tilde b_i
\end{eqnarray}
where $b_i$ and $\tilde b_i$ denote the $i$-th column of the following frame basis and its dual:
\begin{eqnarray}
B&=& E^1E^2 \cdots E^{\kappa},~\quad \label{eq:Bc}\\
\tilde B &=& D^1D^2 \cdots D^{\kappa} \label{eq:tBc}
\end{eqnarray}

2) For the encoder-decoder CNN with skipped connection,
if the following frame conditions are satisfied for all $l\in [\kappa]$:
\begin{eqnarray}\label{eq:PRl2}
 \tilde\Phi^{l}\Phi^{l\top}=\alpha I_{m_{l-1}},~
  \Psi^l \tilde\Psi^{l\top} =    \frac{1}{r(\alpha+1)}I_{rq_{l-1}} 
\end{eqnarray}
then \eqref{eq:PR0} holds,
where $b_i$ and $\tilde b_i$ denote the $i$-th column of the following  frame and its duals:
\begin{eqnarray}\label{eq:Btot}
B^{skp}  \quad ( \in \Rd^{d_0\times (d_\kappa+\sum_{l=1}^\kappa s_l}))
\end{eqnarray}
$$:= \begin{bmatrix} E^1\cdots E^\kappa &E^1\cdots E^{\kappa-1}S^\kappa & \cdots   & E^1S^2& S^1 \end{bmatrix}$$
\begin{eqnarray}\label{eq:tBtot}
\tilde B^{skp}  \quad ( \in \Rd^{d_0\times (d_\kappa+\sum_{l=1}^\kappa s_l}))
\end{eqnarray}
$$:= \begin{bmatrix} D^1\cdots D^\kappa &D^1\cdots D^{\kappa-1}\tilde S^\kappa & \cdots   & D^1\tilde S^2& \tilde S^1 \end{bmatrix}$$
%
\end{proposition}

Furthermore, the following corollary
shows that the total basis and its dual indeed come from multiple convolutional
operations across layers:
\begin{corollary}\label{thm:multi}
If there exist no pooling layers, then the $t$-th block of the frame basis matrix for $t\in[q_l]$ is given by
$$\left[E^{1}\cdots E^{l}\right]_{t}=\left[E^{1}\cdots E^{l-1}S^l \right]_{t}$$
$$ =I_m \circledast \left(\sum_{j_{l-1},\cdots, j_1=1}^{q_{l-1},\cdots,q_1}  \psi_{j_1,1}^l\circledast \cdots \circledast \psi_{t,j_{l-1}}^{l} \right)$$ 
Similarly,
$$\left[D^{1}\cdots D^{l}\right]_{t}=\left[D^{1}\cdots D^{l-1}\tilde S^l \right]_{t}$$
$$ =I_m\circledast \left(\sum_{j_{l-1},\cdots, j_1=1}^{q_{l-1},\cdots,q_1}  \tilde\psi_{j_1,1}^l\circledast \cdots \circledast {\tilde\psi}_{t,j_{l-1}}^{l} \right)$$ 
\end{corollary}

This suggests that  the length of the convolutional filters increases with the depth by cascading multiple convolution
operations across the layers.  
While Proposition~\ref{thm:PR} informs that the skipped connection
increases the dimension of the feature space 
 from $d_\kappa$ to $d_\kappa+\sum_{l=1}^\kappa s_l$, 
Corollary~\ref{thm:multi}  suggest  that the cascaded expression of the filters becomes
 more diverse for the case of encoder-decoder CNNs with skipped connection. Specifically, instead of convolving all $\kappa$ layers of filters, the skipped connection allows
 the combination of subset of filters. All these make  the frame representation from skipped connection more expressive.

\subsection{Expressiveness }

However, to satisfy the frame conditions \eqref{eq:PRl} or \eqref{eq:PRl2},  we need $q_l\geq rq_{l-1}$ so that
the number of output filter channel $q_l$ should increase
exponentially. 
 While this condition can be relaxed when the underlying signal has low-rank Hankel matrix structure \cite{ye2017deep}, the explicit use of the frame condition is still rarely observed. 
Moreover, in contrast to the classical wavelet analysis,
 the perfect reconstruction condition itself is not interesting in neural networks, since the output of the network 
should be different from the input due to the task dependent processing. 

Here, we claim that one of the important roles of using ReLU is that it allows combinatorial basis selection such that
exponentially large number of basis expansion is feasible once the network is trained.
This is in contrast with the standard framelet basis estimation.
For example, for a given target data
$Y = \begin{bmatrix} y^{(1)} & \cdots & y^{(T)} \end{bmatrix}$  and the input data $X= \begin{bmatrix} x^{(1)} & \cdots & x^{(T)} \end{bmatrix}$,
the estimation problem of the frame basis and its dual in Proposition~\ref{thm:PR}
is optimal for the given training data,  but the network is not expressive and 
does not generalize well when the different type of input data is given.
Thus, one of the important requirements  is to allow large number of  expressions that are adaptive to the different
inputs.

Indeed,  ReLU nonlinearity makes the network more expressive. For example, consider a trained two layer encoder-decoder CNN:
\begin{eqnarray}
y = \tilde B \Lambda(x) B^\top x
\end{eqnarray}
where $\tilde B\in \Rd^{d_0\times d_1}$ and $ B\in \Rd^{d_0\times d_1}$ 
and $\Lambda(x)$ is a diagonal matrix with 0, 1 elements that are determined by the ReLU output.
Now, the matrix can be equivalently represented by
\begin{eqnarray}\label{eq:proj}
\tilde B\Lambda(x)B^\top = \sum_{i=1}^{d_1} \sigma_i(x) \tilde b_i b_i^{\top} 
\end{eqnarray}
where $\sigma_i(x)$ refers to the $(i,i)$-th diagonal element of $\Lambda(x)$.
Therefore, depending on the input data $x\in \Rd^{d_0}$,  $\sigma_i(x)$ is either 0 or 1 so that a maximum  $2^{d_1}$ distinct
configurations of the matrix can be represented using \eqref{eq:proj}, which 
is significantly more expressive than using  the single representation with the frame and its dual.
This  observation can be generalized as shown in Theorem~\ref{thm:decexp}.

\begin{theorem}[Expressiveness of encoder-decoder networks]\label{thm:decexp}
Let
\begin{eqnarray}
 \tilde\Upsilon^l= \tilde\Upsilon^l(x) :=  \tilde\Upsilon^{l-1} \tilde\Lambda^{l}(x) D^{l} ,~ \label{eq:UD0} \\
   \Upsilon^{l}=   \Upsilon^{l}(x) := \Upsilon^{l-1} E^{l} \Lambda^{l}(x)  ,~ \label{eq:UE0}
\end{eqnarray}
with $ \tilde\Upsilon^0(x) =I_{d_0}$ and $ \Upsilon^{0}(x) =I_{d_0}$, and
\begin{eqnarray}
 M^l= M^l(x) := S^{l}\Lambda_S^{l}(x)\label{eq:M0} \\
\tilde M^l=\tilde M^l(x) := \tilde \Lambda^{l}(x)\tilde S^{l} \label{eq:tM0} 
\end{eqnarray}
where   $\Lambda^l(x)$ and  $\tilde\Lambda^l(x)$ refer to the  diagonal matrices from ReLU at the $l$-th layer encoder and decoder, respectively,
which have 1 or 0 values;  $\Lambda_S^l(x)$  refers to a similarly defined diagonal matrices from ReLU at the  $l$-th skipped branch
of encoder. 
Then, the following statements are true.

1) Under ReLUs, an encoder-decoder CNN  without skipped connection can be represented by
\begin{eqnarray}\label{eq:feed}
y =  \tilde\Bc(x)\Bc^{\top}(x)x  = \sum_i \langle x, b_i(x) \rangle \tilde b_i(x) 
\end{eqnarray}
where
\begin{eqnarray}\label{eq:Bcx}
 \Bc(x) = \Upsilon^\kappa(x)&,& \tilde \Bc(x) = \tilde\Upsilon^\kappa(x)
\end{eqnarray}
Furthermore,  the maximum number of available linear representation is given by
\begin{eqnarray}\label{eq:nproj}
 N_{rep} = 2^{\sum_{i=1}^{\kappa}d_i-d_\kappa},\quad
 \end{eqnarray}
 
 2) An encoder-decoder CNN  with skipped connection under ReLUs is given by
 \begin{eqnarray}\label{eq:skipnet}
 y = \tilde\Bc^{skp}(x)  \Bc^{skp \top}(x)x  = \sum_i  \langle x, b_i^{skp}(x) \rangle \tilde b_i^{skp}(x) 
 \end{eqnarray}
 where
 $$ \Bc^{skp}(x) := $$
 \begin{eqnarray}\label{eq:Bcxskip}
 \begin{bmatrix} \Upsilon^\kappa & \Upsilon^{\kappa-1}M^\kappa &  \Upsilon^{\kappa-2}M^{\kappa-1} & \cdots & M^1\end{bmatrix}
\end{eqnarray}
 $$\tilde \Bc^{skp}(x) := $$
 \begin{eqnarray}\label{eq:tBcxskip}
 \begin{bmatrix}\tilde \Upsilon^\kappa & \tilde\Upsilon^{\kappa-1}\tilde M^\kappa & \tilde \Upsilon^{\kappa-2}\tilde M^{\kappa-1} & \cdots & \tilde M^1\end{bmatrix}
\end{eqnarray}
Furthermore,  the maximum number of available linear representation is given by
\begin{eqnarray}\label{eq:nproj2}
 N_{rep} = 2^{\sum_{i=1}^{\kappa}d_i-d_\kappa}\times 2^{\sum_{i=1}^\kappa s_k}
 \end{eqnarray}
\end{theorem}

This implies that  the number of representation  increase exponentially with the network
depth, which again confirm the expressive power of the neural network.
Moreover, the skipped connection also significantly increases the expressive power of the encoder-decoder CNN.
Another important consequence of Theorem~\ref{thm:decexp} is that
 the input space $\Xbc$ is partitioned into the maximum $N_{rep}$ non-overlapping
regions so that inputs for each region shares the same linear representation. 


Due to the ReLU, one may wonder whether the cascaded convolutional interpretation of the frame basis in Corollary~\ref{thm:multi} still holds.
A close look of the proof of Corollary~\ref{thm:multi} reveals that  this is still the case. 
Under ReLUs, note that
 $(I_m\circledast \psi_{j,s}^l)(I_m\circledast \psi_{t,j}^{l+1}) = I_m\circledast (\psi_{j,s}^l\circledast\psi_{t,j}^{l+1})$ in Lemma~\ref{lem:identity} should be replaced with  $(I_m\circledast \psi_{j,s}^l)\Lambda_{j}^l(x)(I_m\circledast \psi_{t,j}^{l+1})$
 where $\Lambda_{j}^l(x)$ is a diagonal matrix with 0 and 1 values due to the ReLU. This
 means that the $\Lambda_j^l(x)$ provides spatially varying mask to the convolution filter $\psi_{t,j}^{l+1}$ so that the net effect
 is a convolution with the  the spatially varying filters originated from masked version of $\psi_{t,j}^{l+1}$.
This results in  a spatially variant cascaded convolution, and only change in the interpretation
of Corollary~\ref{thm:multi}   is that the basis and its dual are composed of  {\em spatial variant} cascaded convolution filters.
 Furthermore, the ReLU works to diversify the convolution filters by masking out the various filter coefficients. It is believed that this is another source of expressiveness from the same set of convolutional filters. 


\subsection{Generalizability}


To understand the generalization capability of DNNs,
recent research efforts have been focused on reducing the gap by suggesting different ways of measuring the 
network  capacity \cite{bartlett2017spectrally,neyshabur2018towards}.  
These works
consistently showed the importance of Lipschitz condition for the encoder and decoder parts of the networks.

 More specifically,  we have shown that the neural network representation
varies in  exponentially many different forms depending on  inputs, so one may be concerned that the output might vary drastically
with small perturbation of the inputs.
However, Lipschitz continuity of the neural network prevents such drastic changes.
Specifically, a neural network $F(\Wb,x)$  is Lipschitz continuous, if there exists a constant $K>0$ such that
$$\|F(\Wb,x^{(1)})-F(\Wb,x^{(2)}) \|_2 \leq K \|x^{(1)}-x^{(2)}\|_2 \ .$$
where the Lipschitz constant
$K$ can be obtained by
\begin{eqnarray}\label{eq:K}
K =\sup_{x\in \Xbc}\|D_2 F(\Wb,x)\|_2
\end{eqnarray}
where $D_2F(\Wb,x)$ is the Jacobian with respect to the second variable.
The following proposition shows that the Lipschitz constant of   encoder-decoder CNNs
is closely related to the frame basis and its duals.

\begin{proposition}\label{thm:lipschitz}
The Lipschitz constant for encoder-decoder CNN without skipped connection 
is given by
\begin{eqnarray}\label{eq:lip1}
K= \sup_{x\in \Xbc} {\|\tilde\Bc(x)\Bc(x)^\top\|_2} 
\end{eqnarray}
whereas Lipschitz constant for encoder-decoder CNN with skipped connection 
is given by
\begin{eqnarray}\label{eq:lip2}
K= \sup_{x\in \Xbc} {\|\tilde\Bc^{skp}(x)\Bc^{skp\top}(x)\|_2}
\end{eqnarray}
where  $\Bc(x),\tilde\Bc(x),\Bc^{skp}(x)$ and $\tilde\Bc^{skp}(x)$ are defined in \eqref{eq:Bcx}, \eqref{eq:Bcxskip} and \eqref{eq:tBcxskip}.
\end{proposition}

Recall that the input space $\Xbc$ is partitioned into regions that share the same linear representation.
Therefore, the local  Lipschitz constant within the $p$-th partition is given by
\begin{eqnarray}\label{eq:lip3}
K_{p} &=& \sup_{ z\in \Xbc_p} {\|\tilde\Bc(z)\Bc^{\top}(z)\|_2}  \notag \\
& = & {\|\tilde\Bc(z_p)\Bc^{\top}(z_p)\|_2},\quad \forall z_p \in \Xbc_p 
\end{eqnarray}
for the case of E-D CNN without skipped connections.
Here,
$\Xbc_p$ denotes the $p$-th input space partition,
and the last equality in \eqref{eq:lip3} comes from the fact that every point in $\Xbc_p$ shares the same
linear representation. 
Thus, it is easy to see that the global Lipschitz constant
can be given by
\begin{eqnarray}\label{eq:globalL}
K= \sup_{x\in \Xbc} {\|\tilde\Bc(x)\Bc(x)^\top\|_2} = \sup_p K_p
\end{eqnarray}
Furthermore, Theorem~\ref{thm:decexp} informs that the number of partition is bonded by $N_{rep}$.
Therefore,  \eqref{eq:globalL} suggests that by bounding the local Lipschitz constant within each linear region,
one could control the global Lipschitz constant of the neural network.
Similar observation holds for E-D CNNs with skipped connection.

One of the most important implications of \eqref{eq:globalL} is  that
the expressiveness of the network is not affected by the control of the Lipschitz constant.
This in turn is due to the combinatorial nature of the ReLU nonlinearities, which allows for
 an exponentially large number
of linear representations.


%
%

\subsection{Optimization landscape}

For a given ground truth {\em task map} $f^*:\Xbc\mapsto\Ybc$ and given training data
set $\{(x^{(i)}, y^{(i)})\}_{i=1}^T$ such that $y^{(i)} = f^*(x^{(i)})$,  an encoder-decoder CNN  training problem
can be  formulated to find a neural network parameter weight $\Wb$ by minimizing a specific loss
function.
%
Then, for the  case of $l_2$ loss:
\begin{eqnarray}\label{eq:cost}
C(\Wb) = \frac{1}{2}\sum_{i=1}^T \|F(\Wb,x^{(i)})-y^{(i)}\|^2 \  ,
\end{eqnarray}
Nguyen et al \cite{nguyen2018optimization} showed that 
over-parameterized CNNs can produce zero training errors. Their results
are based on the following key lemma.
\begin{lemma}\cite{nguyen2018optimization}\label{thm:zero}
Consider an encoder-decoder CNN without skipped connection. Then, the Jacobian of the cost
function in \eqref{eq:cost} with respect to
$E^\kappa$ is bounded as
$$\|\nabla_{E^\kappa} C\|_F$$
\begin{eqnarray*}
 \geq \sigma_{\min} (\Xi^\kappa)
 \min_{i\in[T]} \sigma_{\min} \left(\Lambda^\kappa(x^{(i)}) \left(\tilde\Upsilon^{\kappa}(x^{(i)})\right)^\top\right)
 \sqrt{2C(\Wb)}
\end{eqnarray*}
and
$$\|\nabla_{E^\kappa} C\|_F$$
\begin{eqnarray*}
 \leq \sigma_{\max} (\Xi^\kappa)
  \max_{i\in[T]} \sigma_{\max} \left(\Lambda^\kappa(x^{(i)}) \left(\tilde\Upsilon^{\kappa}(x^{(i)})\right)^\top\right)
 \sqrt{2C(\Wb)}
\end{eqnarray*}
where $\sigma_{\min}(A)$ and $\sigma_{\max}(A)$ denote the minimum and maximum
singular value for a matrix $A\in \Rd^{n\times m}$ with $n\geq m$, respectively;
 $\tilde\Upsilon^\kappa$ is defined in \eqref{eq:UD0}, and 
 $\Xi^\kappa $ denotes the feature matrix for the training data
$$\Xi^\kappa = \begin{bmatrix} \xi^{\kappa(1)} & \cdots & \xi^{\kappa(T)}  \end{bmatrix}\quad \in \Rd^{d_\kappa\times T} $$
and $C(\Wb)$ is the  cost in \eqref{eq:cost}.
\end{lemma}

The authors in  \cite{nguyen2018optimization}  further showed that
if every shifted $r$-segment of training  
 samples is not identical to each other and $d_\kappa\geq T$, then $\Xi^\kappa$ has full column rank.
Additionally, if the nonlinearity at the decoder layer is analytic, then they showed that
 $\tilde\Upsilon^\kappa(x)\Lambda^\kappa(x)$ has
almost always full row rank. This implies that both
$ \sigma_{\min} (\Xi^\kappa)$ and $\sigma_{\min} (\Lambda^\kappa(\tilde\Upsilon^\kappa)^\top)$ are non-zero so that
$\left. \nabla_{E^\kappa} C\right|_\Wb=0$ if and only if $y^{(i)}=F(\Wb,x^{(i)})$ for all $i\in [T]$ (that is, the loss becomes zero, i.e. $C(\Wb)=0$).

Unfortunately, this almost always guarantee cannot be used for the ReLU nonlinearities at the decoder layers, since
the ReLU nonlinearity is not analytic.
In this paper, we extend the result of  \cite{nguyen2018optimization} for the encoder-decoder CNN with skipped connection when ReLU nonlinearities are used. In addition to
Lemma~\ref{thm:zero},
the following lemma, which is original, does hold for this case.
\begin{lemma}\label{lem:zero2}
Consider an encoder-decoder CNN with skipped connection. Then, the Jacobian of the cost
function in \eqref{eq:cost} with respect to
$\tilde S^l$ for $l\in [\kappa]$ is bounded as
$$\|\nabla_{\tilde S^l} C\|_F  $$
\begin{eqnarray*}
 \geq \sigma_{\min} (\Gamma^l)
 \min_{i\in[T]} \sigma_{\min} \left(\tilde\Lambda^l(x^{(i)}) \left(\tilde\Upsilon^{l-1}(x^{(i)})\right)^\top\right)
 \sqrt{2C(\Wb)}
\end{eqnarray*}
and
$$\|\nabla_{\tilde S^l} C\|_F $$
\begin{eqnarray*}
 \leq 
 \sigma_{\max} (\Gamma^l)
 \max_{i\in[T]} \sigma_{\max} \left(\tilde\Lambda^l(x^{(i)}) \left(\tilde\Upsilon^{l-1}(x^{(i)})\right)^\top\right)
  \sqrt{2C(\Wb)}
\end{eqnarray*}
where 
 $\Gamma^l $ denotes the feature matrix from the skipped branch
$$\Gamma^l = \begin{bmatrix} \chi^{l(1)} & \cdots & \chi^{l(T)}  \end{bmatrix}\quad \in \Rd^{s_l\times T} $$
and $C(\Wb)$ is the  cost in \eqref{eq:cost}.
\end{lemma}

Lemma~\ref{lem:zero2} leads to the following key results on the optimization landscape for the encoder-decoder
network with skipped connections.
\begin{theorem}
Suppose that there exists a layer $l \in [\kappa]$ such that
\begin{itemize}
\item   skipped features $\chi^{l(1)},\cdots, \chi^{l(T)}$  are linear independent.
\item   $\tilde\Upsilon^{l-1}(x) \tilde\Lambda^l(x)$  has full row rank for all training data $x\in[x^{(1)},\cdots, x^{(T)}]$.
\end{itemize}
Then, $\left. \nabla_{\tilde S^l} C\right|_\Wb=0$ if and only if $y^{(i)}=F(\Wb,x^{(i)})$ for all $i\in [T]$ (that is,  the loss becomes
zero, i.e. $C(\Wb)=0$).
\end{theorem}
\begin{proof}
Under the assumptions, both
$\sigma_{\min} (\Gamma^l)$ and $\sigma_{\min} (\tilde\Lambda^l(\tilde\Upsilon^{l-1})^\top)$ are non-zero. 
Therefore, Lemma~\ref{lem:zero2} leads to the conclusion.
\end{proof}

Note that the proof for the full column rank condition for $\Xi^\kappa$ in  \cite{nguyen2018optimization} is based
on the constructive proof using independency of intermediate features
$ \chi^{l(1)}, \cdots , \chi^{l(T)}$ for all $l\in [\kappa]$.
Furthermore, for the case of ReLU nonlinearities, even when
$\tilde\Upsilon^\kappa(x)\Lambda^\kappa(x)$ does not have full row rank, 
there are  chances that  $\tilde\Upsilon^{l-1}(x)\tilde\Lambda^l(x)$  has full row rank at least one $l\in [\kappa]$.
Therefore, our result has more relaxed assumptions than the optimization landscape results in  \cite{nguyen2018optimization} that relies on
Lemma~\ref{thm:zero}. This again
confirms the advantages of the skipped connection in encoder-decoder networks.

\section{Discussion and Conclusion}

In this paper, we investigate the geometry of encoder-decoder CNN from various theoretical aspects such
as differential topological view, expressiveness, generalization capability and optimization landscape. 
The analysis was feasible thanks to the explicit construction of encoder-decoder CNNs using the deep convolutional framelet expansions. Our analysis showed that the advantages of the encoder-decoder CNNs 
comes from the expressiveness of the encoder and decoder layers, which are originated from the combinatorial nature of ReLU for decomposition and reconstruction frame basis selection. Moreover, the expressiveness of the network
is not affected by controlling Lipschitz constant to improve the generalization capability of the network. 
In addition, we showed that the optimization landscape can be enhanced by the skipped connection.

This analysis coincides with our empirical verification using deep neural networks for various inverse problems.
For example, in a recent work of $k$-space deep learning \cite{han2018k},  we showed that
a neural network for compressed sensing MRI can be more effectively designed
in the $k$-space  domain, since the frame representation is more concise in the Fourier domain.
Similar observation was made in sub-sampled ultrasound (US) imaging \cite{yoon2018efficient}, where we show
that the frame representation in raw data domain is more effective in US so that the deep network
is designed in the raw-date domain rather than image domain.
 These empirical examples clearly showed  that the unified view between signal processing and machine learning as suggested
 in this paper
 can help to improve design and understanding of deep models. 



\section*{Acknowledgements} 

The authors thank to reviewers who gave useful comments.
This work was  supported by the National Research Foundation (NRF) of Korea  grant  NRF-2016R1A2B3008104.



\appendix

\section{Basic Definitions and Lemmas}
\label{ap:basic}

The definition and lemma are from \cite{ye2017deep}, which are included here for self-containment.

Let $x\in \Rd^n$ and $\psi\in\Rd^r$. We further denote $\overline \psi[n]:=\psi[-n]$ as the flipped version of vector whose indices are reversed using periodic boundary condition.
Then, a single-input single-output (SISO) circular convolution of the input $f$ and the filter $\overline \psi$  can be represented in a matrix form:
\begin{eqnarray}\label{eq:SISO}
y = x\circledast \overline\psi &=& \hank_r^n(x) \psi \ ,
\end{eqnarray}
where  $\hank_r^n(x)\in \Rd^{n\times r}$ is a wrap-around  Hankel matrix:
 \begin{eqnarray} \label{eq:hank}
\hank_r^n(x) =\left[
        \begin{array}{cccc}
        x[0]  &   x[1] & \cdots   &   x[r-1]   \\
       x[1]  &   x[2] & \cdots &     x[r] \\
           \vdots    & \vdots     &  \ddots    & \vdots    \\
              x[n-1]  &   x[n] & \cdots &   x[r-2] \\
        \end{array}
    \right] 
    \end{eqnarray}
    
By convention, when we use the circular convolution $u\circledast v$ between the different length vectors $u$ and $v$, we assume that
  the period of the
  convolution follows that of the longer vector. 
Furthermore, when we construct a $n\times r$ Hankel matrix using a small size vector, e.g $\hank_r^n(z)$ with $z\in \Rd^d$ with $d<n$,
 then we implicitly imply that  appropriate number of zeros is added to $z$  to construct a Hankel matrix.
 This ensures the following commutative relationship:
\begin{eqnarray}\label{eq:commute}
x\circledast \psi =\hank_r^n(x) \overline\psi = \hank_n^n(\psi) \overline x = \psi \circledast x
\end{eqnarray}
 
Similarly, multi-input multi-output (MIMO) convolution for the $p$-channel
input $Z=[z_1,\cdots,z_p]$ and $q$-channel output can be represented  by
\begin{eqnarray}\label{eq:MIMO}
y_i = \sum_{j=1}^{p} z_j\circledast \overline\psi_{i,j},\quad i=1,\cdots, q
\end{eqnarray}
where $p$ and $q$ are the number of  input and output channels, respectively;
$\overline\psi_{i,j} \in \Rd^r$ denotes the length $r$- filter that convolves the $j$-th channel input to compute its contribution to 
the
$i$-th output channel. 
By defining the MIMO filter kernel $\Phi$ as follows:
\begin{eqnarray*}
\Psi = \begin{bmatrix} \Psi_1 \\ \vdots \\ \Psi_p \end{bmatrix} \, \quad \mbox{where} \quad \Psi_j =  \begin{bmatrix} \psi_{1,j}  & \cdots & \psi_{q,j} \end{bmatrix} 
\end{eqnarray*}
the corresponding matrix representation of the MIMO convolution is then given by 
\begin{eqnarray*}
Y 
&=& \sum_{j=1}^p \hank_r^n(z_j) \Psi_j  = \hank_{r|p}^n\left(Z\right) \Psi \label{eq:multifilter}
\end{eqnarray*}
where    
$\hank_{r|p}^n\left(Z\right)$ is  an {extended Hankel matrix}  by stacking  $p$ Hankel matrices side by side: 
\begin{eqnarray}\label{eq:ehank}
\hank_{r|p}^n\left(Z\right)  := \begin{bmatrix} \hank_r^n(z_1) & \hank_r^n(z_2) & \cdots & \hank_r^n(z_p) \end{bmatrix} \  
\end{eqnarray}
where $z_i$ denotes the $i$-th column of $Z$.
The following basic properties  of Hankel matrix are from \cite{yin2017tale}
 \begin{lemma}\label{lem:calculus}
For a given $f\in \Rd^n$,  let  $\hank_r^n(f) \in \Rd^{n\times r}$ denote the associated Hankel matrix.
Then, for any vectors $u,v\in \Rd^n$ and  any Hankel matrix $F := \hank_r^n(f)$, we have
\begin{eqnarray}\label{eq:inner}
  u^{\top} F v  = u^{\top} \left( f \circledast \overline v \right) = f^{\top} \left( u\circledast v \right) = \langle f, u\circledast v \rangle
\end{eqnarray}
where  $\overline v[n]:=v[-n]$ denotes the flipped version of the vector $v$.
\end{lemma}

\section{Derivation of the matrix representations}
\label{ap:matrix}

Using definition in  \eqref{eq:defconv},  we have
 \begin{eqnarray*}
(\Phi^l \circledast \psi_{j,k}^l)^\top \xi_k^{l-1}  &=&\begin{bmatrix}  \phi_1^{l\top}  (\xi_k^{l-1}\circledast \overline\psi_{j,k})  \\ \vdots\\  \phi_{m_l}^{l\top}  (\xi_k^{l-1}\circledast \overline\psi_{j,k}^l)\end{bmatrix} \\
&=& \Phi^{l\top} (\xi_k^{l-1}\circledast \overline\psi_{j,k}^l)  
 \end{eqnarray*}
 where the first equality comes from \eqref{eq:inner}, i.e.
 $$ (\phi_i^l\circledast \psi_{j,k}^l)^\top\xi_k^{l-1}=\phi_i^{l\top}  (\xi_k^{l-1}\circledast \overline\psi_{j,k}^l) .$$
 Therefore,
\begin{eqnarray*}
\sigma\left(\sum_{k=1}^{q_{l-1}} 
(\Phi^l \circledast \psi_{j,k}^l)^\top \xi_k^{l-1}\right) &=& \sigma\left(\sum_{k=1}^{q_{l-1}}\Phi^\top (\xi_k^{l-1}\circledast \overline\psi_{j,k}^l) \right)\\
&=& \xi_j^{l}
\end{eqnarray*}
This proves the encoder representation.

For the decoder part, note that
$$(\tilde\Phi^l \circledast \tilde\psi_{j,k}^l) \tilde\xi_k^{l}  $$
\begin{eqnarray*}
&=& \begin{bmatrix}\hank_{{m_{l-1}}}^{r}(\tilde\phi_1^l)\overline{\tilde\psi}_{j,k}^l & \cdots & \hank_{{m_{l-1}}}^{r}(\tilde\phi_{m_{l}}^l)\overline{\tilde\psi}_{j,k}^l \end{bmatrix} \tilde\xi_k^{l}\\
&=&\hank_{{m_{l-1}}}^{m_{l-1}}(\tilde\psi_{j,k}^l)\begin{bmatrix}\overline{\tilde\phi}_1^l & \cdots & \overline{\tilde\phi}_{m_{l}}^l\end{bmatrix}  \tilde\xi_k^l\\
&=&\hank_{{m_{l-1}}}^{m_{l-1}}(\tilde\psi_{j,k}^l)\overline{\tilde\Phi^l \tilde\xi_k^l}\\
&=&\tilde\Phi^l \tilde\xi_k^l \circledast \tilde\psi_{j,k}^l
 \end{eqnarray*}
 where we use the commutativity in \eqref{eq:commute} for the second equality.
 Therefore,
\begin{eqnarray*}
\sigma\left(\sum_{k=1}^{q_{l}} 
(\tilde\Phi^l \circledast \tilde\psi_{j,k}^l) \tilde\xi_k^{l} \right)&=& \sigma\left(\sum_{k=1}^{q_{l}}\tilde\Phi^l \tilde\xi_k^l \circledast \tilde\psi_{j,k}^l \right)\\
&=& \tilde\xi_j^{l-l}
\end{eqnarray*}
This proves the decoder representation.
The proof for the skipped branch is a simple corollary by using the identity pooling operation.

\section{Proof of Proposition~\ref{thm:embedding}}
\label{ap:whitney}

The proof is basically same as in Theorem~3 in \cite{shen2018differential}. Only modification is to replace a linear surjective map with a Lipschitz continuous map.
Specifically, we need to show that  for a continuous function
$h: \Xbc \mapsto \Rd^p$, there is a smooth embedding $\tilde h: \Xbc \mapsto \Rd^p$ that satisfies
\begin{eqnarray}\label{eq:step1}
 \|g \circ h (x) - g \circ \tilde h(x)\| \leq \epsilon
\end{eqnarray}
Due to the Lipschitz continuity, there exist $K\geq 0$ such that
\begin{eqnarray}\label{eq:step2}
 \|g \circ h (x) - g \circ \tilde h(x)\|_2 \leq K \| h(x) -\tilde h(x) \|_2 
\end{eqnarray}
Now, according to  the weak Whitney Embedding Theorem \cite{whitney1936differentiable,tu2011introduction},
 for any  $\epsilon' > 0$, if $p >2 \dim \Xbc$, then there exists a smooth embedding $\tilde h: \Xbc \mapsto \Rd^p$ 
such that
$$\| h(x) -\tilde h(x) \|_2  \leq \epsilon' $$
By plugging  this into \eqref{eq:step2} and using the inequalities between norm, we can prove \eqref{eq:step1}. Q.E.D.

 \section{Proof of Proposition~\ref{thm:PR}}
 
 First, consider an encoder-decoder CNN without skipped connection.
The $(s,t)$ block of the $l$-th layer encoder-decoder pair is given by
$$ \left[D^{l}E^{l\top}\right]_{s,t}=\sum_{j=1}^{q_l} (\tilde\Phi^{l}\circledast \tilde\psi^{l}_{s,j}) (\Phi^l\circledast \psi^l_{j,t})^\top$$
 \begin{eqnarray*}
&=&  \sum_{j=1}^{q_l}  \begin{bmatrix} \tilde \phi_1^{l} \circledast \tilde\psi^{l}_{s,j} & \cdots & \tilde\phi_{m_l}^{l} \circledast 
 \tilde\psi^{l}_{s,j}\end{bmatrix}\\
&&\quad \cdot
\begin{bmatrix} ( \phi_1^l \circledast \psi^l_{j,t})^\top \\ \vdots\\ (\phi_{m_l}^l \circledast \psi^l_{j,t})^\top\end{bmatrix} \\
&=& \sum_{i=1}^{m_l} \sum_{j=1}^{q_l}\hank_{r}^{m_{l-1}}(\tilde \phi_i^{l}) \overline{\tilde\psi^{l}}_{s,j}\overline {\psi}_{j,t}^{l\top} \hank_{r}^{{m_{l-1}}\top}(\phi_i^l) \\
&=& \frac{1}{\alpha r}\sum_{i=1}^{m_l} \hank_{r}^{m_{l-1}}(\tilde \phi_i^{l}) \delta_{s,t}\hank_{r}^{m_{l-1}\top}(\phi_i^l) \\
&=& \frac{\delta_{s,t}}{\alpha r}\sum_{k=1}^{r}\sum_{i=1}^{m_l} P_k \tilde \phi_i^{l} \phi_i^{l\top} P_k^\top  \\
&=&\frac{\delta_{s,t}}{r}\sum_{k=1}^{r}P_k  P_k^\top = \delta_{s,t}I_{m_{l-1}}
\end{eqnarray*}
where  $\delta_{s,t}=1$ for $s=t$ or zero otherwise, $P_k$ denotes the periodic shift by $k$, and the fourth and the sixth equalities come from the 
frame condition for the filters and pooling layers. 
This results in  $D^lE^{l\top}=I_{d_{l-1}}$.
Now, note that
\begin{eqnarray*}
\sum_i \langle b^i, x\rangle \tilde b_i  &=& \tilde B B^{\top} x  \\
&=&D^1\cdots D^\kappa E^{\kappa \top} \cdots E^{1\top}x
\end{eqnarray*}
By applying $D^lE^{l\top}=I_{d_{l-1}}$ from $l=\kappa$ to $l=1$, we conclude the proof.

Second, consider an encoder-decoder CNN with skipped connections.
In this case, 
$$ \left[D^{l}E^{l\top}+\tilde S^l S^{l\top} \right]_{s,t}=$$
$$\sum_{j=1}^{q_l} (\tilde\Phi^{l}\circledast \tilde\psi^{l}_{s,j}) (\Phi^l\circledast \psi^l_{j,t})^\top +  (I_{m_{l-1}}\circledast \tilde\psi^{l}_{s,j}) (I_{m_{l-1}}\circledast \psi^l_{j,t})^\top$$
Using the same trick,  we have
\begin{eqnarray*}
\sum_{j=1}^{q_l} (\tilde\Phi^{l}\circledast \tilde\psi^{l}_{s,j}) (\Phi^l\circledast \psi^l_{j,t})^\top 
=\frac{\alpha}{\alpha+1}  \delta_{s,t}I_{d_{l-1}}
\end{eqnarray*}
and for the second part, we have
\begin{eqnarray*}
\sum_{j=1}^{q_l} (I_{m_{l-1}}\circledast \tilde\psi^{l}_{s,j}) (I_{m_{l-1}}\circledast \psi^l_{j,t})^\top
=\frac{1}{\alpha+1}  \delta_{s,t}I_{d_{l-1}}
\end{eqnarray*}
Therefore, 
$ \left[D^{l}E^{l\top}+\tilde S^l S^{l\top} \right]_{s,t}=\delta_{s,t}I_{d_{l-1}}$ and  
\begin{eqnarray}\label{eq:id}
D^{l}E^{l\top}+\tilde S^l S^{l\top} = I_{d_{l-1}},\quad \forall l\in [\kappa]
\end{eqnarray}

Now, we derive the basis representation in \eqref{eq:Btot} and \eqref{eq:tBtot}.
Let 
$$z:= \begin{bmatrix} \xi^\kappa  \\ \chi^{\kappa} \\ \vdots \\ \chi^1 \end{bmatrix}$$
Then, using the construction in \eqref{eq:feature} without considering ReLUs, we can easily show that
\begin{eqnarray*}
\xi^\kappa &=& E^{\kappa \top} \cdots E^{1\top} x \\
\chi^{\kappa} &=& S^{\kappa \top} E^{(\kappa-1)\top}\cdots E^{1\top} x  \\
 & \vdots & \\
\chi^{1} &=& S^{1 \top}  x 
\end{eqnarray*}
Accordingly, we have
$$z = B^{skp\top} x$$
where $B^{skp}$ is defined in \eqref{eq:Btot}.
Now, from the definition of the decoder layer in \eqref{eq:sum}, we have
\begin{eqnarray*}
\tilde \xi^0 &=& D^1 \tilde \xi^1 + \tilde S^1  \chi^1 \\
\tilde \xi^1 &=& D^2 \tilde \xi^2 + \tilde S^2  \chi^2 \\
&\vdots & \\
\tilde \xi^{\kappa-1} &=& D^\kappa \tilde \xi^\kappa + \tilde S^\kappa  \chi^\kappa \\
\tilde \xi^\kappa &=& \xi^\kappa 
\end{eqnarray*}
Accordingly, we have
$$\tilde \xi^0 = \tilde B^{skp} z = \tilde B B^{skp\top} x$$
where $\tilde B^{skp}$ is defined in \eqref{eq:tBtot}.
This proves the representation.

Now, for any $l\in [\kappa]$, note that
$$D^1\cdots D^l E^{l\top}  \cdots E^{1\top} + D^1\cdots D^{l-1}\tilde S^l S^l E^{(l-1)\top} \cdots E^{1\top}$$
$$=D^1\cdots D^{l-1}  (D^l E^{l\top}  + \tilde S^l S^l ) E^{(l-1)\top} \cdots E^{1\top}$$
$$=D^1\cdots D^{l-1} E^{(l-1)\top}  \cdots E^{1\top}$$
where we use \eqref{eq:id}.
By applying this recursively from $l=\kappa$ to $l=1$, we can easily show that
$$\tilde B^{skp} B^{skp}=I_{d_0}.$$
This concludes the proof.

 \section{Proof of Corollary~\ref{thm:multi}}
 
 \begin{lemma}\label{lem:identity}
 For given vectors $v,w\in \Rd^m$, we have
  \begin{eqnarray}
( I_m \circledast v) (I_m\circledast w)  = I_m \circledast (w\circledast v)
 \end{eqnarray}
 \end{lemma}
 \begin{proof}
 By definition in \eqref{eq:defconv}, we have
\begin{eqnarray*}
 I_m \circledast v &=& \begin{bmatrix} e_1^m \circledast v& e_2^m \circledast v & \cdots &  e_m^m \circledast v
 \end{bmatrix} \\
 &=& \begin{bmatrix}
 v[0] & v[m-1] & \cdots & v[1]\\ v[1] & v[0] & \cdots  & v[2]  \\
 \vdots& \vdots & \ddots & \vdots \\
  v[m-1] & v[m-2] &\cdots & v[0]
 \end{bmatrix}
 \end{eqnarray*}
 Accordingly, for any vector $u\in \Rd^m$, we have
 $$ ( I_m \circledast v) u  $$
 \begin{eqnarray*}
=  \begin{bmatrix}
 v[0] & v[m-1] & \cdots & v[1]\\ v[1] & v[0] & \cdots  & v[2]  \\
 \vdots& \vdots & \ddots & \vdots \\
  v[m-1] & v[m-2] &\cdots & v[0]
 \end{bmatrix} u = u \circledast v
 \end{eqnarray*}
 Therefore, we have
 $$( I_m \circledast v) (I_m\circledast w)  $$
  \begin{eqnarray*}
&=& ( I_m \circledast v)\begin{bmatrix} e_1^m \circledast w& e_2^m \circledast w & \cdots &  e_m^m \circledast w
 \end{bmatrix} \\
 &=&\begin{bmatrix} e_1^m \circledast w \circledast v& e_2^m \circledast w \circledast v& \cdots &  e_m^m \circledast w\circledast v
 \end{bmatrix}\\
&=& I_m \circledast (w \circledast v)
 \end{eqnarray*}
This concludes the proof.
 \end{proof}
 
 \begin{lemma}\label{lem:mult}
 \begin{eqnarray*}
 \left[E^{l}E^{l+1}\right]_{s,t}=   I_m \circledast \left(\sum_{j=1}^{q_l} \psi_{j,s}^l\circledast \psi_{t,j}^{l+1}\right)  
 \end{eqnarray*}
 \end{lemma}
 \begin{proof}
 Since there is no pooling layers,  we have $m_l=m, \forall l\in [\kappa]$.
Therefore, using Lemma~\ref{lem:identity},  the $(s,t)$ block  is given by
 \begin{eqnarray*}
 \left[E^{l}E^{l+1}\right]_{s,t}&=&\sum_{j=1}^{q_l} (I_{m}\circledast \psi^{l}_{j,s}) (I_{m}\circledast \psi^{l+1}_{t,j})\\
&=&  \sum_{j=1}^{q_l}  I_m\circledast (\psi_{j,s}^l\circledast \psi_{t,j}^{l+1})  \\
&=&  I_m \circledast \left(\sum_{j=1}^{q_l} \psi_{j,s}^l\circledast\psi_{t,j}^{l+1}\right)  
\end{eqnarray*}
Q.E.D.
\end{proof}

\subsection{Proof}
If there exists no pooling layers, then $E^l=S^l$ and $D^l = \tilde S^l$ for all $l\in [\kappa]$.
Therefore, we only show the case for $E^l$.
We will prove by induction.
For $\kappa =1$,  using Lemma~\ref{lem:mult}, we have
 \begin{eqnarray*}
 \left[E^{1}E^{2}\right]_{t}
&=&  I_m \circledast \left(\sum_{j=1}^{q_1} \psi_{j,1}^1\circledast\psi_{t,j}^{2}\right)  
\end{eqnarray*}
Suppose that this is true for $k$. Then, for $k+1$, we have
$$\left[E^{1}\cdots E^{k}E^{k+1}\right]_{t}=$$
$$ \sum_{j_{k}=1}^{q_{k}}   I_m\circledast \left(\sum_{j_{k-1},\cdots, j_1=1}^{q_{k-1},\cdots,q_1}  \psi_{j_1,1}^l\circledast\psi_{j_2,j_1}^{2} \cdots \circledast \psi_{j_k,j_{k-1}}^{k} \right)$$ 
$$ \cdot I_m \circledast (\psi_{t,j_k}^{k+1}) $$
$$=I_m\circledast \left(\sum_{j_{k},\cdots, j_1=1}^{q_{k},\cdots,q_1}  \psi_{j_1,1}^l\circledast \cdots \circledast \psi_{j_k,j_{k-1}}^{k} \circledast \psi_{t,j_k}^{k+1} \right) $$ 
where the last equality comes from Lemma~\ref{lem:identity}. This concludes the proof of the first part.
The second part of the corollary is a simple repetition of the proof.

\section{Proof of Theorem~\ref{thm:decexp}}
\label{ap:decexp}

First, we will prove the case for the encoder-decoder CNN without skipped connection.
Note that the main difference of the encoder-decoder CNN
without skipped connection from the convolutional framelet expansion in \eqref{eq:PR0} is the
existence of the ReLU for each layer.  This can be readily implemented
using a diagonal matrix $\Lambda^l(x)$ or $\tilde\Lambda^l(x)$ with 1 and 0 values in front of the $l$-th layer, whose diagonal
values are determined by the ReLU output. Note that the reason we put a dependency $x$ in $\Lambda^l(x)$ is that
the ReLU output is a function of input $x$.
Therefore, by adding  $\Lambda^l(x)$ or $\tilde\Lambda^l(x)$ between layers in \eqref{eq:Bc} and \eqref{eq:tBc}, we can readily
obtain the expression \eqref{eq:feed}. 
Then, for $\kappa$-layer encoder decoder CNN, the number of
diagonal elements for the ReLU matrices are
$\sum_{l=1}^\kappa {d_{l}} -d_\kappa$ 
where the last subtraction comes from the existence of one ReLU layer at the $\kappa$ layer.
Since these diagonal matrix  $\Lambda^l(x)$ or $\tilde\Lambda^l(x)$ can have either 0 or 1 values, the total number of representation
becomes $2^{\sum_{l=1}^\kappa {d_{l}} -d_\kappa}$. 

Second,  we will prove the case for the encoder-decoder CNN with skipped connection.
Note that the main difference of the encoder-decoder CNN
with skipped connection from the convolutional framelet expansion using basis in \eqref{eq:Btot} and \eqref{eq:tBtot} is the
existence of the ReLU for each layer.  This can be again readily implemented
using a diagonal matrix  $\Lambda^l(x)$ or $\tilde\Lambda^l(x)$. Therefore, by adding  $\Lambda^l(x)$ or $\tilde\Lambda^l(x)$ between layers in \eqref{eq:Btot} and \eqref{eq:tBtot}, we can readily
obtain the expression \eqref{eq:skipnet}. 
Now, compared to the encoder-decoder CNN without skipped connection, there exists additional
$\kappa$ ReLU layer in front of the skipped branch from each encoder layers.
Since the dimension of the $l$-th skipped branch output is $s_l=m_{l-1}q_l$, the $l$-th ReLU layer in front of skipped branch
can have $2^{s_l}$ representation. By considering all these cases, we can arrive at \eqref{eq:nproj2}.
Q.E.D.

\section{Proof of Proposition~\ref{thm:lipschitz}}
Here, we derive the condition by assuming a skipped connection.
The case without skipped connection can be derived as a special case of this.


\begin{lemma}\label{lem:chain}
If there is a skipped connection, then for any $l\in [\kappa]$ we have
\begin{eqnarray*}
\frac{\partial \tilde\xi^{l-1}}{\partial x} 
&=&\tilde\Lambda^{l}D^{l}\frac{\partial \tilde\xi^{l}}{\partial x}+ \tilde\Lambda^{l}\tilde S^{l}\frac{\partial \chi^{l}}{\partial x} \\
\frac{\partial \xi^{l}}{\partial x}
&=&  \tilde\Lambda^{l} E^{l\top}\frac{\partial \xi^{l-1}}{\partial x}
\end{eqnarray*}
where $\tilde\Lambda^l$ denotes the diagonal matrix representing the derivative of ReLU operation and  
\begin{eqnarray}
\frac{\partial \chi^{l}}{\partial x}&=&
  \tilde\Lambda_S^{l} S^{l\top}\frac{\partial \xi^{l-1}}{\partial x}  \label{eq:dS}
\end{eqnarray}
\end{lemma}
\begin{proof}
With the skipped connection, only $\tilde\xi^{l-1}$ is a function of $\tilde\xi_{l}$ and $s_l$. Then, the proof is a simple application
of the chain rule. The reason we replace $\dot{\tilde\Lambda}^l$ with $\tilde\Lambda^l$ is that the derivative of ReLU operation
also results in a diagonal matrix with the same 0 and 1 values depending on the ReLU output.
\end{proof}

\begin{lemma}\label{lem:chain2}
For any $k\in [\kappa]$,  we have
\begin{eqnarray}\label{eq:induction0}
\frac{\partial \tilde\xi^{0}}{\partial x} = \tilde\Upsilon^k\frac{\partial \tilde\xi^{k}}{\partial x}+\sum_{i=1}^{k} \tilde\Upsilon^{i-1} \tilde M^{i}M^{i\top} \Upsilon^{i-1\top} 
\end{eqnarray}
where $ \tilde\Upsilon^l,\Upsilon^{l}$ and $M^l $ are defined in \eqref{eq:UD0}, \eqref{eq:UE0} and \eqref{eq:M0}, respectively.
\end{lemma}
\begin{proof}
We will prove this by induction. When $k=1$, using Lemma~\ref{lem:chain}, we have
\begin{eqnarray*}
\frac{\partial \tilde\xi^{0}}{\partial x} &=& \tilde\Lambda^{1}D^{1}\frac{\partial \tilde\xi^{1}}{\partial x}+  \tilde\Lambda^{1}\tilde S^{1}\tilde\Lambda^1 S^{1\top}\\
&=& \tilde\Upsilon^{1}\frac{\partial \tilde\xi^{1}}{\partial x}+ \tilde\Upsilon^{0}\tilde M^1 M^{1\top}\Upsilon^{0\top} \ .
\end{eqnarray*}
Now, assuming that this is true for $k$, we will prove it for $k+1$.
Using Lemma~\ref{lem:chain}, we have
\begin{eqnarray*}
\frac{\partial \tilde\xi^{k}}{\partial x} &=& \tilde\Lambda^{k+1}D^{k+1}\frac{\partial \tilde\xi^{k+1}}{\partial x} + \tilde\Lambda^{k+1}\tilde S^{k+1}\frac{\partial \chi^{k+1}}{\partial x} \\
&=& \tilde\Lambda^{k+1}D^{k+1}\frac{\partial \tilde\xi^{k+1}}{\partial x} + \tilde M^{k+1}M^{k+1\top}\frac{\partial \xi^{k}}{\partial x}\\
&=& \tilde\Lambda^{k+1}D^{k+1}\frac{\partial \tilde\xi^{k+1}}{\partial x} +  \tilde M^{k+1}M^{k+1\top} \Upsilon^{k\top}
\end{eqnarray*}
where we use \eqref{eq:dS} for the second equality.
By plugging this in \eqref{eq:induction0}, we conclude the proof.
\end{proof}

\subsection{Proof}

By applying Lemma~\ref{lem:chain2} up to $k=\kappa$, we have
\begin{eqnarray*}
\frac{\partial \xi^{0}}{\partial x} &=& \tilde\Upsilon^{\kappa}\frac{\partial\tilde\xi^{\kappa}}{\partial x}+\sum_{i=1}^{\kappa}  \tilde\Upsilon^{i-1} \tilde M^{i} M^{i\top} \Upsilon^{i-1\top} \notag\\
&=&  \tilde\Upsilon^{\kappa} \Upsilon^{\kappa}+\sum_{i=1}^{\kappa}  \tilde\Upsilon^{i-1} \tilde M^{i} M^{i\top} \Upsilon^{i-1\top} \\
&=&\tilde\Bc^{skp}(x)\Bc^{skp\top}(x)
\end{eqnarray*}
Using \eqref{eq:K}, we have
\begin{eqnarray*}
K &=& \sup_{x\in \Xbc} \| D_2 F(\Wb,x)\|_2 =  \sup_{x\in \Xbc} \left\| \frac{\partial \xi^{0}}{\partial x}\right\|_2 \\
&=& \sup_{x\in \Xbc} \|\tilde\Bc^{skp}(x)\Bc^{skp\top}(x)\|_2 \\
&=& \sup_{z\neq 0, x,z\in \Xbc} \frac{\|\tilde\Bc^{skp}(x)\Bc^{skp\top}(x)z\|_2}{\|z\|_2} 
\end{eqnarray*}
where the last equality comes from the definition of spectral norm.
The proof for the case without skipped connection is a simple corollary.
This concludes the proof.

\section{Proof of Lemma~\ref{lem:zero2}}
\label{ap:landscape}
%
%
%
%
%
%

\begin{lemma}\label{lem:chainap}
For any $l\in [\kappa]$,  we have
\begin{eqnarray}
\frac{\partial \tilde\xi^{0}}{ \partial \tilde S^{l}} = \tilde\Upsilon^{l-1}\tilde\Lambda^l \left(\chi^{l\top}\otimes I_{d_{l-1}}
\right)
\end{eqnarray}
\end{lemma}
\begin{proof}
Using the chain rule and noting the $\dot{\tilde\Lambda}^l(x)=\tilde\Lambda^l(x)$, we have
$$\frac{\partial \tilde\xi^{0}}{ \partial \tilde S^{l}} = \tilde\Upsilon^{l-1}\tilde\Lambda^l\frac{\partial \tilde S^{l} \chi^{l}}{\partial \tilde S^{l}}.$$
Furthermore, note that $\mathrm{vec}(AXB)=\left(B^\top \otimes A\right)\mathrm{vec}(X)$ where $\mathrm{vec}(\cdot)$ denotes the vectorization operation.  Accordingly, we have
\begin{eqnarray*}
\frac{\partial \tilde S^l \chi^l}{\partial \tilde S^l} &:=& \frac{\partial \tilde S^l \chi^l}{\partial \mathrm{vec}(\tilde S^l)}\\
&=& \frac{\partial (\chi^{l\top}\otimes I_{d_{l-1}})\mathrm{vec}(\tilde S^l)}{\partial \mathrm{vec}(\tilde S^l)}\\
&=&\chi^{l\top}\otimes I_{d_{l-1}}
\end{eqnarray*}
This concludes the proof.
\end{proof}
\begin{lemma}\label{lem:sigma}
For $A\in \Rd^{n\times m}$ and $B\in \Rd^{m\times p}$ with $n\geq m$, 
\begin{eqnarray*}
\sigma_{\min}(A) \|B\|_F\leq \|AB\|_F \leq \sigma_{\max}(A) \|B\|_F
\end{eqnarray*}
\end{lemma}
\begin{proof}
See Lemma H.3 in \cite{nguyen2018optimization}.
\end{proof}

\subsection{Proof}

For the cost function \eqref{eq:cost}, we have
\begin{eqnarray*}
-\nabla_{\tilde S^l} C  &=& \sum_{i=1}^T \left(\frac{\partial \tilde S^l \chi^{(i)l}}{\partial \tilde S^l}\right)^\top \left(y^{(i)}-F(\Wb,x^{(i)})\right)
\end{eqnarray*}
\begin{eqnarray*}
&=&  \sum_{i=1}^T(\chi^{(i)l}\otimes I_{d_{l-1}})\tilde\Lambda^l(x^{(i)}) \left(\tilde\Upsilon^{l-1}(x^{(i)})\right)^\top \left(y^{(i)}-F(\Wb,x^{(i)})\right) \\
&=& \left(\Gamma^{l}\otimes I_{d_{l-1}}\right) \Db^l \Rb
\end{eqnarray*}
where
$$\Db^l := $$
\begin{eqnarray}\label{eq:Db}
\begin{bmatrix}  \tilde\Lambda^l(x^{(1)}) \left(\tilde\Upsilon^{l-1}(x^{(1)})\right)^\top & \cdots &
0 \\ \vdots & \ddots & \vdots \\ 0 & \cdots & \tilde\Lambda^l(x^{(T)}) \left(\tilde\Upsilon^{l-1}(x^{(T)})\right)^\top \end{bmatrix}
\end{eqnarray}
and 
$$\Rb = \begin{bmatrix} y^{(1)}-F(\Wb,x^{(1)}) \\ \vdots \\ y^{(T)}-F(\Wb,x^{(T)}) \end{bmatrix}$$
Because $\left(\Gamma^{l}\otimes I_{d_{l-1}}\right) \in \Rd^{s_ld_{l-1}\times Td_{l-1}}$ and
$s_l\geq T$, Lemma~\ref{lem:sigma} informs that
$$\|\nabla_{\tilde S^l} C\|_F$$
\begin{eqnarray*}
&\leq&  \sigma_{\max}\left(\Gamma^{l}\otimes I_{d_{l-1}}\right) \|\Db^l \Rb^l\|_F \\
&=& \sigma_{\max}\left(\Gamma^{l}\right) \|\Db^l \Rb^l\|_F \\
\end{eqnarray*}
Furthermore,
$\Db^l \in \Rd^{Td_{l-1} \times T d_0}$ and $d_{l-1}\geq d_0$,
Lemma~\ref{lem:sigma} informs that
 we have $\|\Db^l \Rb^l\|_F\leq \sigma_{\max}(\Db^l) \|\Rb^l\|_F$.
 Furthermore, from the definition in \eqref{eq:Db}, we have
$$\sigma_{\max}(\Db^l) =  \max_{i\in[T]} \sigma_{\max} \left(\tilde\Lambda^l(x^{(i)}) \left(\tilde\Upsilon^{l-1}(x^{(i)})\right)^\top\right).$$

Therefore,
$$\|\nabla_{\tilde S^l} C\|_F$$
\begin{eqnarray*}
&\leq & \sigma_{\max}\left(\Gamma^{l}\right) \max_{i\in[T]} \sigma_{\max} \left(\tilde\Lambda^l(x^{(i)}) \left(\tilde\Upsilon^{l-1}(x^{(i)})\right)^\top\right)\| \Rb\|_F \\
\end{eqnarray*}
Similarly, using Lemma~\ref{lem:sigma}, we have
$$\|\nabla_{\tilde S^l} C\|_F$$
\begin{eqnarray*}
&\geq &\sigma_{\min}\left(\Gamma^{l}\right) \min_{i\in[T]} \sigma_{\min} \left(\tilde\Lambda^l(x^{(i)}) \left(\tilde\Upsilon^{l-1}(x^{(i)})\right)^\top\right)\| \Rb\|_F \\
\end{eqnarray*}
Using
$\| \Rb\|_F^2 =\sum_{i=1}^T \|y^{(i)}-F(\Wb,x^{(i)})\|^2=2C(\Wb)$,  we conclude the proof.

\end{document}